\documentclass[9pt,twocolumn,twoside]{IEEEtran}

\usepackage{subfigure}
\usepackage{amsfonts}
\usepackage{amsmath}
\usepackage{amssymb}
\usepackage{graphicx}
\usepackage{algorithm}
\usepackage{algorithmic}

\newtheorem{problem}{\textbf{Problem}}

\newtheorem{proposition}{\textbf{Proposition}}

\newcommand{\eq}[1]{\begin{equation}\label{#1}}
\newcommand{\en}{\end{equation}}

\title{Graph-based classification of multiple observation sets}

\author{\begin{tabular}{cc} Effrosyni Kokiopoulou & Pascal Frossard\thanks{This work has been mostly performed while the first author was with the Signal Processing Laboratory (LTS4) of EPFL. It has been partly supported by the Swiss National Science Foundation, under grant NCCR IM2.}\\
ETHZ & Ecole Polytechnique F\'ed\'erale de Lausanne (EPFL) \\
Seminar for Applied Mathematics & Signal Processing Laboratory - LTS4 \\
CH - 8092 Z\"urich & CH - 1015 Lausanne\\
\tt{effrosyni.kokiopoulou@sam.math.ethz.ch} & \tt{pascal.frossard@epfl.ch} \end{tabular} }

\begin{document}
%\topmargin=0mm
%\ninept
%
\maketitle
\thispagestyle{empty}

\begin{abstract}
We consider the problem of classification of an object given multiple observations that possibly include different transformations. The possible transformations of the object generally span a low-dimensional manifold in the original signal space. We propose to take advantage of this manifold structure for the effective classification of the object represented by the observation set. In particular, we design a low complexity solution that is able to exploit the properties of the data manifolds with a graph-based algorithm. Hence, we formulate the computation of the unknown
label matrix as a smoothing process on the manifold under the
constraint that all observations represent an object  of one
single class. It results into a discrete optimization problem,
which can be solved by an efficient and low complexity algorithm.
We demonstrate the performance of the proposed graph-based algorithm in the
classification of sets of multiple images. Moreover, we show
its high potential in video-based face recognition, where it outperforms
state-of-the-art solutions that fall short of exploiting the
manifold structure of the face image data sets.
\end{abstract}

\begin{IEEEkeywords}
\noindent Graph-based classification, multiple observations sets, video face recognition, multi-view
object recognition.
\end{IEEEkeywords}

\section{Introduction}
\label{sec:intro}

Recent years have witnessed a dramatic growth of the amount of digital data that is produced by sensors or computers of all sorts. That creates the need for efficient processing and analysis algorithms in order to extract the relevant information contained in these datasets. In particular, it commonly happens that multiple observations of an object are captured at different time instants or under
different geometric transformations. For instance, a moving object
may be observed over a time interval by a surveillance camera (see
Fig. \ref{fig:multobs:video}) or under different viewing angles by
a network of vision sensors (see Fig. \ref{fig:multobs:visionnet}). This typically produces a large
volume of multimedia content that lends itself as a valuable
source of information for effective knowledge discovery and
content analysis. In this context, classification methods should be able to exploit the diversity of
the multiple observations in order to provide increased
classification accuracy \cite{Stauffer.03}.

\begin{figure}[t]
\begin{center}
%\mbox{
     \subfigure[Video frames of a moving object]{\label{fig:multobs:video}\includegraphics[width=3in]{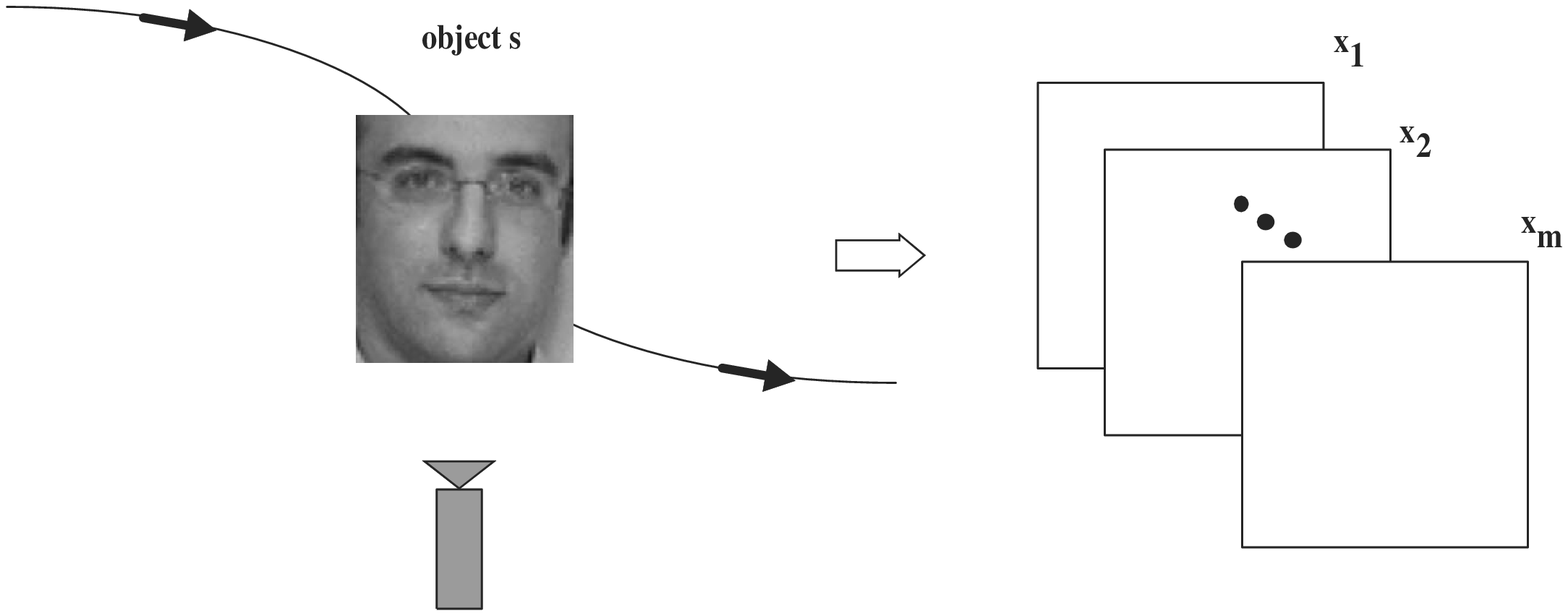}}
     \subfigure[Network of vision sensors]{\label{fig:multobs:visionnet}\includegraphics[width=2in]{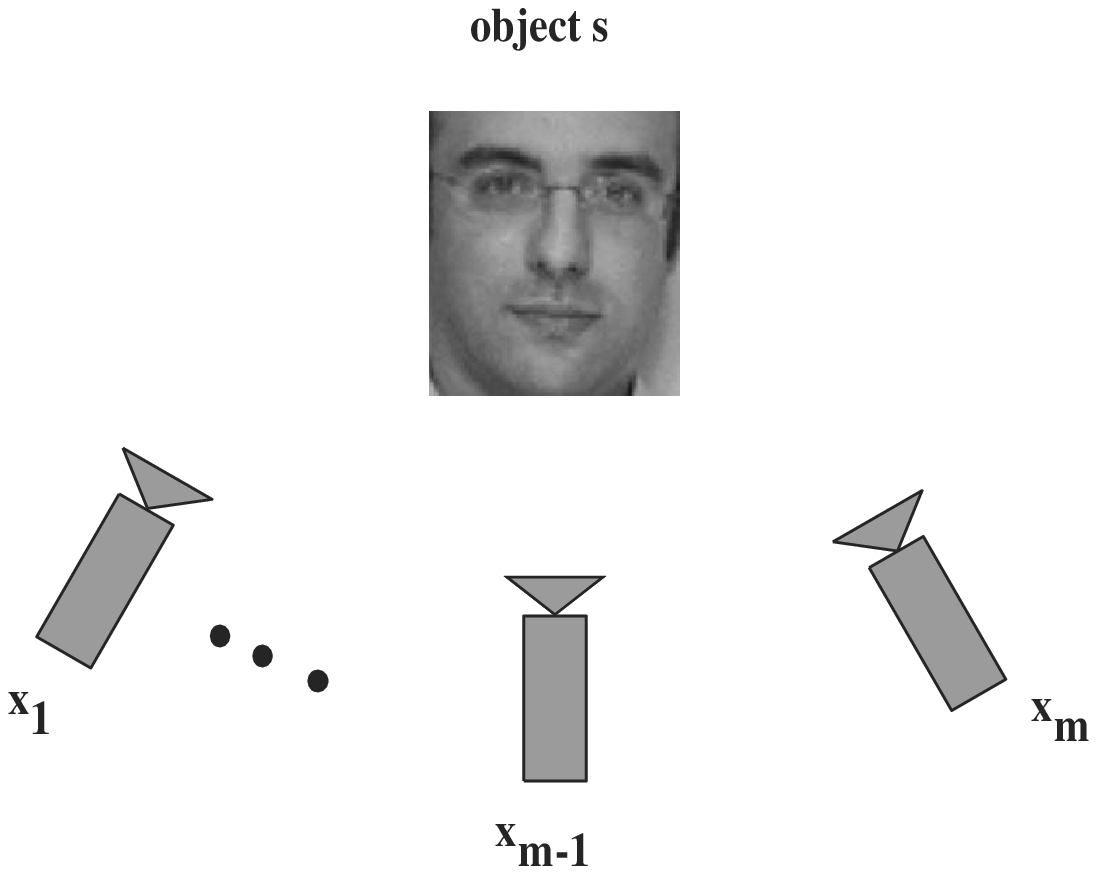}}
%     }
\end{center}
\caption{Typical scenarios of producing multiple observations of
an object.} \label{fig:multobs}
\end{figure}

We build on our previous
work \cite{KokPirFro-ICPR.08} and we focus here on the pattern classification problem with
multiple observations. We further assume that observations are
produced from the same object under different transformations, so that they all lie on the same low-dimensional manifold. 
%this problem can
%be identified as a particular case of semi-supervised learning
%\cite{SSL-book}. Semi-supervised learning refers to the type of
%learning where the test unlabelled data are available in the
%training phase; the challenge is to exploit this extra information
%in order to improve the classification performance. Moreover, all unlabelled examples typically belong to the same
%unknown class in our problem of classification of multiple observation sets.
We propose a novel graph-based algorithm built on label
propagation \cite{ZhouBousNLalWestScolk.03}. Label propagation
methods typically assume that the data lie on a low dimensional
manifold living in a high dimensional space. They rely upon the
\emph{smoothness assumption}, which states that if two data
samples $x_1$ and $x_2$ are close, then their labels $y_1$ and
$y_2$ should be close as well. The main idea of these methods is
to build a graph that captures the geometry of this manifold as
well as the proximity of the data samples. The labels of the test
examples are derived by ``propagating" the labels of the labelled
data along the manifold, while making use of the smoothness
property. We exploit the specificities of our particular
classification problem and constrain the unknown labels to
correspond to one single class. This leads to the formulation of a
discrete optimization problem that can be optimally solved by a
simple and low complexity algorithm.

We apply the proposed algorithm to the classification of sets of
multiple images in handwritten digit recognition, multi-view object
recognition or video-based face recognition. In particular, we show 
the high potential of our graph-based method for efficient classification of images that
belong to the same data manifold. For example, the proposed
solution outperforms state-of-the-art subspace or statistical
classification methods in video-based face recognition and object 
recognition from multiple image sets. Hence, this paper establishes 
new connections between graph-based algorithms and the problems of classification of
multiple image sets or video-based face recognition, where the proposed solutions are certainly very promising.

The paper is organized as follows. We first formulate the  problem
of classification of multiple observation sets in Section
\ref{sec:problemdefinition}. We introduce our graph-based algorithm inspired by
label propagation in Section \ref{sec:graphapproach}. Then we demonstrate the
performance of the proposed classification method for handwritten
digit recognition, object recognition and video-based face recognition in Sections
\ref{sec:digits}, \ref{sec:objrec} and \ref{sec:videoface}, respectively.

\section{Problem definition}\label{sec:problemdefinition}

We address the problem of the classification of multiple
observations of the same object, possibly with some
transformations. In particular, the problem is to assign multiple
observations of the test pattern/object $s$ to a single class of
objects. We assume that we have $m$ transformed observations of
$s$ of the following form
\[
x_i = U(\eta_i)s, ~i=1,\ldots,m,
\]
where $U(\eta)$ denotes a (geometric) transformation operator with
parameters $\eta$, which is applied on $s$. For instance, in the
case of visual objects, $U(\eta)$ may correspond to a rotation,
scaling, translation, or perspective projection of the object. We
assume that each observation $x_i$ is obtained by applying a
transformation $\eta_i$ on $s$, which is different from its peers
(i.e., $\eta_i \neq \eta_j$, for $i \neq j$). The problem is to
classify $s$ in one of the $c$ classes under consideration, using
the multiple observations $x_i,~i=1,\ldots,m$.

Assume further that the data set is organized in two parts $X = \{
X^{(l)}, X^{(u)}\}$, where $X^{(l)} =  \{x_1,x_2,\ldots,x_l\}
\subset \mathbb{R}^d$ and $X^{(u)} = \{x_{l+1},\ldots,x_n \}
\subset \mathbb{R}^d$, where $n = l + m$. Let also $\mathcal{L} =
\{1,\ldots,c \}$ denote the label set. The $l$ examples in
$X^{(l)}$ are labelled $\{y_1,y_2,\ldots,y_l \},~y_i \in
\mathcal{L}$, and the $m$ examples in $X^{(u)}$ are unlabelled.
The classification problem can be formally defined as follows.

\vskip 0.1in
\begin{problem}\label{prb:problemdef}
%Denote by $c$ the number of classes and $\mathcal{L}$ the class
%label set.
Given a set of labelled data $X^{(l)}$, and a set of unlabelled
data $X^{(u)} \triangleq \{ x_j = U(\eta_j)s, ~j=1,\ldots,m \}$
that correspond to multiple transformed observations of $s$, the
problem is to predict the correct class $c^\ast$ of the original
pattern $s$.
\end{problem}
\vskip 0.1in

One may view Problem \ref{prb:problemdef} as a special case of semi-supervised
learning \cite{SSL-book}, where the unlabelled data $X^{(u)}$ represent the multiple observations with the extra constraint that all
unlabelled data examples belong to the same (unknown) class. The problem then resides in estimating the single unknown class, while generic semi-supervised learning problems attribute the test examples to different classes. 

%This problem is a particular case of semi-supervised learning \cite{SSL-book}, which generally consists in predicting the labels of $X^{(u)}$, based on the knowledge of the data points (both $X^{(l)}$ and $X^{(u)}$) and the labels of the labelled points. Note that in the generic scenario of semi-supervised learning, the test examples may belong to different classes. The above problem however presents an important additional constraint, where all the observations belong to the same class. Thus, 

\section{Graph-based classification} \label{sec:graphapproach}

\subsection{Label propagation} \label{sec:LP}

\begin{figure}[t]
\begin{center}
\includegraphics[width=2.5in]{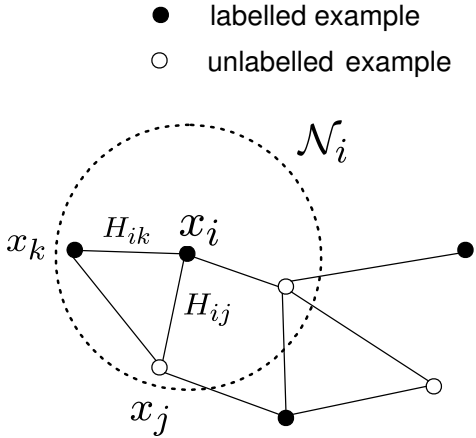}
\end{center}
\caption{Typical structure of the $k$-NN graph. $\mathcal{N}_i$
represents the neighborhood of the sample $x_i$.}
\label{fig:graph}
\end{figure}

We propose in this section a novel method to solve Problem \ref{prb:problemdef}, 
which is inspired by label propagation \cite{ZhouBousNLalWestScolk.03}. The label propagation algorithm is based on a \emph{smoothness assumption}, which states that if $x_1$ and $x_2$ are
close by, then their corresponding labels $y_1$ and $y_2$ should be close as well. Denote by $\mathcal{M}$ the set of matrices with
nonnegative entries, of size $n \times c$. Notice that any matrix $M \in \mathcal{M}$ provides a labelling of the data set by
applying the following rule: $y_i = \max_{j=1,\ldots,c} M_{ij}$. We denote the initial label matrix as $Y \in \mathcal{M}$ where
$Y_{ij} = 1$ if $x_i$ belongs to class $j$ and 0 otherwise. The label propagation algorithm first forms the $k$ nearest neighbor
($k$-NN) graph defined as
\[
\mathcal{G}=(\mathcal{V},\mathcal{E}),
\]
where the vertices $\mathcal{V}$ correspond to the data samples
$X$. An edge $e_{ij} \in \mathcal{E}$ is drawn if and only if
$x_j$ is among the $k$ nearest neighbors of $x_i$. %or vice versa.

It is common practice to assign weights on the edge set of
$\mathcal{G}$. One typical choice is the Gaussian weights
\eq{eq:Gaussian} H_{ij} =
\begin{cases}
\exp(-\frac{\|x_i - x_j\|^2}{2\sigma^2})  & \textrm{when} ~(i,j) \in \mathcal{E}, \\
0                   & \textrm{otherwise}.
\end{cases}
\en

The similarity matrix $S \in R^{n \times n}$ is further defined as
\begin{equation}\label{eq:Smatrix}
S = D^{-1/2} H D^{-1/2},
\end{equation}
where $D$ is a diagonal matrix with entries $D_{ii} = \sum_{j=1}^n
H_{ij}$. See also Fig. \ref{fig:graph} for a schematic
illustration of the $k$-NN graph and related notation.

Next, the algorithm computes a real valued $M^\ast \in
\mathcal{M}$ based on which the final classification is performed
using the rule $y_i = \max_{j=1,\ldots,c} M_{ij}^\ast$. This is
done via a regularization framework with a cost function
defined as
\begin{eqnarray}\label{eq:regularization}
\mathcal{U}(M) &=& \frac{1}{2} \Big( \sum_{i,j=1}^n H_{ij}
\|\frac{1}{\sqrt{D_{ii}}} M_i - \frac{1}{\sqrt{D_{jj}}} M_j \|^2 +
\nonumber \\
&& \mu \sum_{i=1}^n \|M_i-Y_i\|^2 \Big) ,
\end{eqnarray}
where $M_i$ denotes the $i$th row of $M$. The computation
of $M^\ast$ is done by solving the quadratic optimization problem
$M^\ast = \arg \min_{M \in \mathcal{M}} \mathcal{U}(M)$.

Intuitively, we are seeking an $M^\ast$ that is smooth along the
edges of similar pairs $(x_i,x_j)$ and at the same time close to
$Y$ when evaluated on the labelled data $X^{(l)}$. The first term
in (\ref{eq:regularization}) is the \emph{smoothness} term and the
second is the \emph{fitness} term.

Notice that when two examples $x_i$ and $x_j$ are similar (i.e.,
the weight $H_{ij}$ is large) minimizing the smoothness term in
(\ref{eq:regularization}) results in $M$ being smooth across
similar examples. Thus, similar data examples will likely share
the same class label. It can be shown
\cite{ZhouBousNLalWestScolk.03} that the solution to problem
(\ref{eq:regularization}) is given by \eq{eq:LPsolution} M^\ast =
\beta (I-\alpha S)^{-1} \mu Y, \en where $\alpha =
\frac{1}{1+\mu}$ and $\beta = \frac{\mu}{1+\mu}$.

Finally, several other variants of label propagation have been
proposed in the past few years. We mention for instance, the
method of \cite{ZhuGhahramani.02} and the variant of label
propagation that was inspired from the Jacobi iteration algorithm
\cite[Ch. 11]{SSL-book}. Finally, it is interesting to note that
there have also been found connections to Markov random walks
\cite{SzuJaakkola.02} and electric networks \cite{ZhuGhaLaf.03}. Note finally that label propagation is probably the most representative algorithm among the graph-based methods for semi-supervised learning.

\subsection{Label propagation with multiple observations}

We propose now to build on graph-based algorithms to solve the
problem of classification of multiple observation sets. In general,
label propagation assumes that the unlabelled examples come from
different classes. As Problem \ref{prb:problemdef} presents the
specific constraint that all unlabelled data belong to the same
class, label propagation does not fit exactly the definition of
the problem as it falls short of exploiting its special structure.
Therefore, we propose in the sequel a novel graph-based algorithm,
which (i) uses the smoothness criterion on the manifold in order
to predict the unknown class labels and (ii) at the same time, it
is able to exploit the specificities of Problem
\ref{prb:problemdef}.

We represent the data labels with a 1-of-$c$ encoding, which permits to form a binary label matrix of size $n \times c$, whose
$i$th row encodes the class label of the $i$th example. The class
label is basically encoded in the position of the nonzero element.

Suppose now that the correct class for the unlabelled data is the
$p$th one. In this case, we denote by $Z_{p} \in R^{n \times c}$
the corresponding label matrix. Note that there are $c$
such label matrices; one for each class hypothesis. Each \textit{class-conditional label matrix}
$Z_{p}$ has the following form
\begin{equation}\label{eq:Z}
Z_{p} = \left[
\begin{array}{c}
  Y_l \in R^{l \times c}\\
  \hline
  \\
  \textbf{1} e_{p}^\top \in R^{m \times c} \\
\end{array}
\right] \in R^{n \times c},
\end{equation}
where $e_{p} \in R^{c}$ is the $p$th canonical basis vector and
$\textbf{1} \in R^{m}$ is the vector of ones. Fig.
\ref{fig:matrixZ} shows schematically the structure of matrix
$Z_{p}$. The upper part corresponds to the labelled examples and
the lower part to the unlabelled ones. $Z_{p}$ holds the labels of
all data samples, assuming that all unlabelled examples belong to
the $p$th class. Observe that the $Z_{p}$'s share the first part
$Y_l$ and differ only in the second part.

\begin{figure}[t]
\begin{center}
\includegraphics[width=2in]{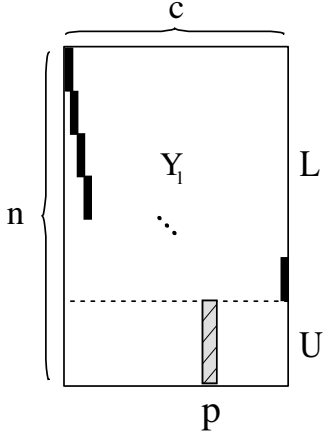}
\end{center}
\caption{Structure of the class-conditional label matrix $Z_{p}$.}
\label{fig:matrixZ}
\end{figure}

Since all unlabelled examples share the same label, the class labels have a special structure that reflects the
special structure of Problem \ref{prb:problemdef}, as outlined in our previous work \cite{KokPirFro-ICPR.08}. We could then express the unknown
label matrix $M$ as,
\begin{equation}\label{eq:M}
M = \sum_{p=1}^c \lambda_{p} Z_{p},~~ Z_{p} \in R^{n \times c},
\end{equation}
where $Z_{p}$ is given in (\ref{eq:Z}), $\lambda_{p} \in \{0,1 \}$
and
\begin{equation}\label{eq:lambda}
\sum_{p=1}^c \lambda_{p} = 1.
\end{equation}
In the above, $\lambda = [\lambda_1,\ldots,\lambda_c]$ is the
vector of linear combination weights, which are discrete and sum
to one. Ideally, $\lambda$ should be sparse with only one nonzero
entry pointing to the correct class.

The classification problem now resides in estimating the proper
value of $\lambda$. We rely on the smoothness assumption and we
propose the following objective function
\begin{eqnarray}\label{eq:regularization2}
\tilde{\mathcal{Q}}(M(\lambda)) = \frac{1}{2} \Big( \sum_{i,j=1}^n
H_{ij} \|\frac{1}{\sqrt{D_{ii}}} M_i - \frac{1}{\sqrt{D_{jj}}} M_j
\|^2 \Big),
\end{eqnarray}
where the optimization variable now becomes the $\lambda$ vector. Notice that the fitting term in Eq. (\ref{eq:regularization}) is not
needed anymore due to the structure of the $Z$ matrices. Furthermore, we observe that the optimization parameter $\lambda$ is implicitly
represented in the above equation through $M$, defined in eq. (\ref{eq:M}). 

In the above, $M_i$ (resp. $M_j$) denotes the $i$th (resp. $j$th)
row of $M$. In the case of normalized similarity matrix, the above criterion
becomes
\begin{eqnarray}\label{eq:objfun}
\mathcal{Q}(M(\lambda)) = \frac{1}{2} \sum_{i,j=1}^n S_{ij} \| M_i -
M_j \|^2,
\end{eqnarray}
where $S$ is defined as in (\ref{eq:Smatrix}). It can be seen that the objective function directly relies on the smoothness assumption. When two examples $x_i$, $x_j$ are nearby (i.e., $H_{ij}$ or $S_{ij}$ is large), minimizing $\tilde{Q}(\lambda)$ and
$Q(\lambda)$ results in class labels that are close too. The following proposition now shows the explicit dependence of $Q$ on $\lambda$.

\begin{proposition}\label{prop:objfun}
Assume the data set is split into $l$ labelled examples $X^{(l)}$
and $m$ unlabelled examples $X^{(u)}$, i.e., $X = [X^{(l)},
X^{(u)}]$. Then, the objective function (\ref{eq:objfun}) can be
written in the following form,
\begin{equation}\label{eq:objfunlambda}
Q(\lambda) = C + \frac{1}{2} \sum_{i \leq l,  j > l} S_{ij} \|Y_i
- \lambda\|^2 + \frac{1}{2} \sum_{i > l,  j \leq l} S_{ij} \|Y_j -
\lambda\|^2
\end{equation}
where $C = \sum_{i \leq l,  j \leq l} S_{ij} \|Y_i - Y_j\|^2$. \\

\end{proposition}
\begin{proof}
From equation (\ref{eq:objfun}) observe that
\begin{eqnarray*}
\mathcal{Q}(\lambda) &=& \underbrace{\frac{1}{2} \sum_{i,j \leq
l}^n S_{ij} \| M_i - M_j \|^2}_{Q_1}
+ \underbrace{\frac{1}{2} \sum_{i,j > l}^n S_{ij} \| M_i - M_j \|^2}_{Q_2}   \\
&&+ \underbrace{\frac{1}{2} \sum_{i \leq l, j > l}^n S_{ij} \| M_i
- M_j \|^2}_{Q_3}  \\
&& + \underbrace{\frac{1}{2} \sum_{i > l, j \leq l}^n S_{ij} \|
M_i - M_j \|^2}_{Q_4}.
\end{eqnarray*}

We consider the following cases
\begin{description}
\item[(i)] $i \leq l$ and $j \leq l$: both data examples
$x_i$ and $x_j$ are labelled. Then, $M_i = (\sum_{p=1}^c
\lambda_{p}) Y_i = Y_i$, due to the special structure of the $Z$
matrices (see (\ref{eq:Z})) and also due to the constraint from Eq.
(\ref{eq:lambda}). Similarly, $M_j = Y_j$. This results in $Q_1 =
\frac{1}{2} \sum_{i, j \leq l} S_{ij} \|Y_i - Y_j\|^2 = C$, which
is a constant term and does not depend on $\lambda$.

\item[(ii)] $i > l$ and $j > l$: both data samples $x_i$ and
$x_j$ are unlabelled. In this case, $M_i = \lambda$ and $M_j =
\lambda$, again due to (\ref{eq:Z}). Therefore the second term
$Q_2$ is zero.

\item[(iii)] $i \leq l$ and $j > l$: $x_i$ is labelled and
$x_j$ is unlabelled. In this case, $M_i = Y_i$ and $M_j =
\lambda$. This results in $Q_3 = \frac{1}{2} \sum_{i \leq l,  j >
l} S_{ij} \|Y_i - \lambda\|^2$.

\item[(iv)] $i > l$ and $j \leq l$ is analogous to the case (iii)
above, where the roles of $x_i$ and $x_j$ are switched. Thus, $Q_4
= \frac{1}{2} \sum_{i> l, j \leq l} S_{ij} \|Y_j - \lambda\|^2$.
\end{description}
Putting the above facts together yields Eq. (\ref{eq:objfunlambda}).
\end{proof}

The above proposition suggests that only the interface between
labelled and unlabelled examples matters in determining the
smoothness value of a candidate label matrix $M$, or equivalently the solution vector $\lambda$. We use this observation in order to design an efficient graph-based classification algorithm that is described below.

\subsection{The MASC algorithm}

\begin{algorithm}[tb]
\caption{The MASC algorithm} \label{Algo:MASC}
\begin{algorithmic} [1]
\STATE \textbf{Input}: \\
$X \in \mathbb{R}^{d \times n}$: data examples. \\
$m$: number of observations.\\
$l$: number of labelled data.
\STATE \textbf{Output}: \\
$\hat{p}$: estimated unknown class.
\STATE \textbf{Initialization}: \\
\STATE Form the $k$-NN graph $\mathcal{G=(V,E)}$. \STATE Compute
the weight matrix $H \in \mathbb{R}^{n \times n}$ and the diagonal
matrix $D$, where $D_{i,i} = \sum_{j=1}^n H_{ij}$. \STATE Compute
$S = D^{-1/2} H D^{-1/2}$. \FOR{$p = 1 : c$} \STATE $M = \left[
\begin{array}{c}
  Y_l \\
  \hline
  \textbf{1} e_{p}^\top  \\
\end{array}
\right]$ \STATE $q(p) = \sum_{i \leq l, j > l} S_{ij} \| M_i - M_j
\|^2 + \sum_{i > l,  j \leq l} S_{ij} \|M_i - M_j\|^2$. \ENDFOR
\STATE $\hat{p} = \arg \min_{p} q(p)$
\end{algorithmic}
\end{algorithm}

We propose in this section a simple, yet effective graph-based algorithm for the classification of multiple observations from the same class. Based on Proposition \ref{prop:objfun} and ignoring the constant
term, we need to solve the following optimization problem
\begin{center}
\fbox{\makebox{
\begin{tabular}{l}
Optimization problem: \textbf{OPT} \\
$\min_{\lambda}  \sum_{i \leq l,  j > l} S_{ij} \|Y_i -
\lambda\|^2 +  \sum_{i > l,  j \leq l} S_{ij} \|Y_j -
\lambda\|^2 $ \\
subject to \\
%%\quad $M = \sum_{p=1}^{c} \lambda_{p} Z_{p}$, \\
\quad $ \lambda_{p} \in \{0,1 \}$, $p = 1,\ldots,c$, \\
\quad $ \sum_{p = 1}^c \lambda_{p} = 1$. \\
\end{tabular}
}}
\end{center}
Intuitively, we seek the class that corresponds to the smoothest
label assignment between labelled and unlabelled data. Observe
that the above problem is a discrete optimization problem due to
the constraints imposed on $\lambda$, that can be collected in a
set $\Lambda$, where
\[
\Lambda = \{\lambda \in R^{c \times 1}: ~\lambda_{p} \in \{0,1\},
p = 1,\ldots,c, \sum_{p=1}^c \lambda_p = 1 \}.
\]
Interestingly, the search space $\Lambda$ is small. In particular,
it consists of the following $c$ vectors:
\begin{eqnarray*}
&&[1, 0,\ldots,0,\ldots, 0] \\
&&[0, 1,\ldots,0,\ldots, 0] \\
&&\ldots \\
&&[0, 0,\ldots,1,\ldots, 0] \\
&&[0, 0,\ldots,0,\ldots, 1].
\end{eqnarray*}
Thus, one may solve OPT by enumerating all above possible
solutions and pick the one $\lambda^\ast$ that minimizes
$Q(\lambda)$. Then, the position of the nonzero entry in
$\lambda^\ast$ yields the estimated unknown class. We call this
algorithm \textbf{MA}nifold-based \textbf{S}moothing under
\textbf{C}onstraints (MASC) and we show its main steps in
Algorithm \ref{Algo:MASC}. The MASC algorithm has a complexity
that is linear with the number of classes, and quadratic with the
number of samples.The construction of $k$-NN graph (lines 4-6)
scales as O($n^2$). Once the graph has been constructed, the
enumeration of all possible solutions scales as O($c$). We
conclude that the total computational cost is O($n^2 + c$).

%which is much lower than the complexity O($c \cdot n^2$) of the modified
%TSVM solution proposed before when the number of classes
%increases.

\section{Classification of multiple images sets}
%\label{sec:digits}

\subsection{Handwritten digit classification}
\label{sec:digits}

\begin{figure*}[t]
\begin{center}
\mbox{
     \subfigure[Binary digits]{\label{fig:binary:digits}\includegraphics[width=3in]{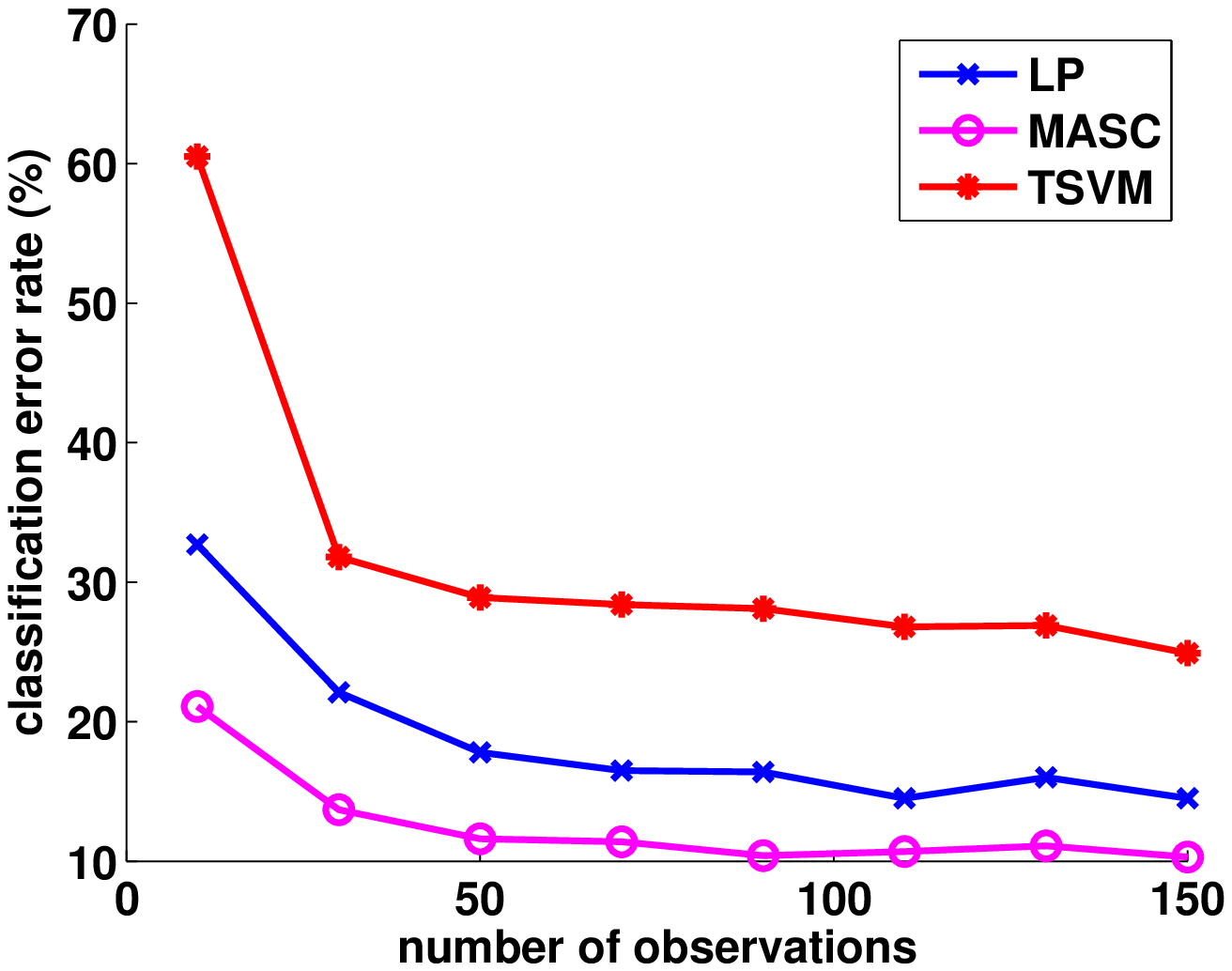}}
     \subfigure[USPS digits]{\label{fig:usps:digits}\includegraphics[width=3in]{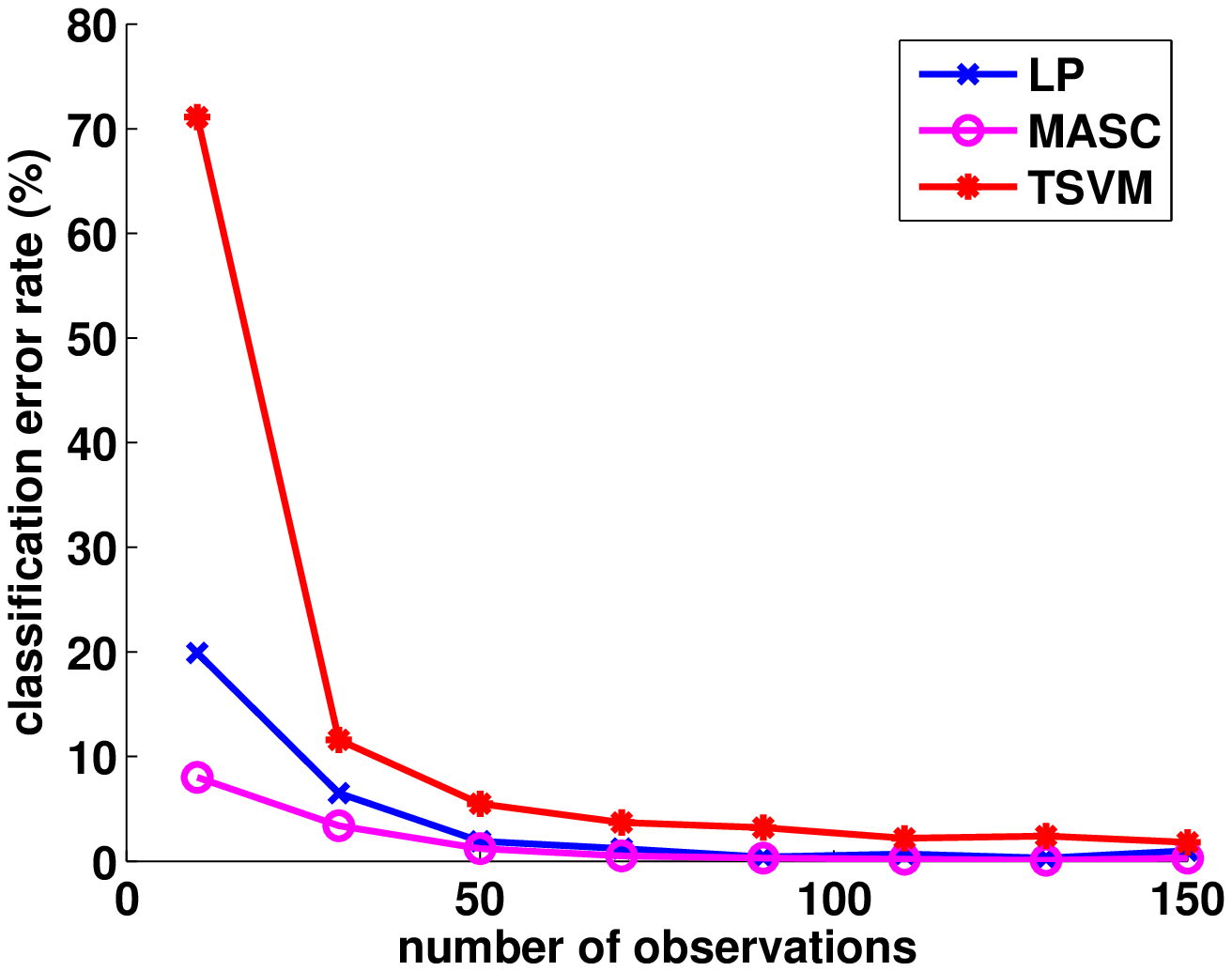}}
     }
\end{center}
\caption{Classification results measured on two different data sets.} \label{fig:digits}
\end{figure*}

We evaluate the performance of the proposed MASC algorithm with
respect to label propagation, in the context of handwritten
digit classification. Multiple transformed images of the same
digit class form a set of observations, which we want to assign in
the correct class. We use two different data sets for our
experimental evaluation; (i) a handwritten digit image
collection\footnote{http://www.cs.toronto.edu/$\sim$roweis/data.html}
and (ii) the USPS handwritten digit image collection. The first
collection contains 20 $\times$ 16 bit binary images of ``0"
through ``9", where each class contains 39 examples. The USPS
collection contains 16 $\times$ 16 grayscale images of digits and
each class contains 1100 examples.

Robustness to pattern transformations is a very
important property of the classification of multiple observations. Transformation invariance can
be reinforced into classification algorithms by augmenting the labelled examples with the so-called \textit{virtual samples},
denoted hereby as $X^{(\textrm{vs})}$ (see \cite{PozdBengio.06} for a similar approach). The virtual samples are essentially data
samples that are generated artificially, by applying transformations to the original data samples. They are given the
class labels of the original examples that they have been generated from, and are treated as labelled data. By including the
virtual samples in the data set, any classification algorithm becomes more robust to transformations of the test examples. We therefore adopt this strategy in the proposed methods and we include $n_{\textrm{vs}}$ virtual samples $X^{(\textrm{vs})}$ in our original data set that is finally written as $X = \{ X^{(l)},
X^{(\textrm{vs})} , X^{(u)} \}$.

We compare the classification performance of the MASC algorithm
with the label propagation (LP) method. In LP, the estimated class
is computed by majority voting on the estimated class labels
computed in Eq. (\ref{eq:LPsolution}). In our experiments, we use the
same $k$-NN graph in combination with the Gaussian weights from Eq.
(\ref{eq:Gaussian}) in both LP and MASC methods. In order to
determine the value of the parameter $\sigma$ in Eq. (\ref{eq:Gaussian}) we adopt
the following process; we pick randomly 1000 examples, compute
their pairwise distances and then set $\sigma$ equal to half of
its median.

We first split the data sets into
training and test sets by including 2 examples per class in the
training set and the remaining are assigned to the test set. Each
training sample is augmented by 4 virtual examples generated by
successive rotations of it, where each rotation angle is sampled
regularly in $[-40^\circ,40^\circ]$. This interval has been chosen
to be sufficiently small in order to avoid the confusion of digits
'6' and '9'. Next, in order to build the unlabelled set
$X^{(u)}$ (i.e., multiple observations) of a certain class, we
choose randomly a sample from the test set of this class and then
we apply a random rotation on it by a random (uniformly sampled)
angle $\theta \in [-40^\circ,40^\circ]$.

%To the best of our knowledge, there is no effective methodology
%for reliable model selection when the number of training examples
%is small. Hence, as is common practice, the parameters of methods
%were set based on the methods' performance on the test set.

The number of nearest neighbors was set to $k=5$ for both binary
digit collection and the USPS data set, in both methods. These
values of $k$ have been obtained by the best performance of LP on
the test set. We try different sizes of the unlabelled set (i.e.,
multiple observations), namely $m = [10:20:150]$ (in MATLAB
notation). For each value of $m$, we report the average
classification error rate across 100 random realizations of
$X^{(u)}$ generated from each one of the 10 classes. Thus, each
point in the plot is an average over 1000 random experiments.

% We also compare the graph-based algorithms with the TSVM method, discussed in
% Section \ref{sec:TSVMproposed}. For the SVM implementation, we use
% the SPIDER machine learning library for Matlab that is publically
% available\footnote{http://www.kyb.tuebingen.mpg.de/bs/people/spider/main.html}.
% The standard SVM used in line 11 of Algorithm \ref{Algo:TSVM},
% consists of a nonlinear SVM with an RBF kernel. We rather opt for
% a nonlinear type of SVM, due to the manifold structure of the data
% (caused by the pattern transformations). In this case, the
% hyperparameters of the TSVM are the width $t$ of the RBF kernel
% and the soft margin parameter $C$. Since the number of training
% data examples is small, the hyperparameters are selected based on
% the test error. We used the following two-dimensional grid for the
% hyperparameters,
% \[
% \{(t_i,C_j), t_i \in [10^{-2}: 10^{0.5} :10], C_j \in [10^{-1} :
% 10^{0.5} : 100] \}.
% \]
% This resulted in $t_1 = 0.3162$ and $C_1 = 10$ for the binary
% digits data set and $t_1 = 0.3162$ and $C_1 = 31.6228$ for the
% USPS data.

Figures \ref{fig:binary:digits} and \ref{fig:usps:digits} show the
results over the binary digits and the USPS digits image
collections, respectively. Observe first that increasing the
number of observations gradually improves the classification error
rate of both methods. This is expected since more observations of a
certain pattern give more evidence, which in turn results in
higher confidence in the estimated class label. Finally, observe that 
the proposed MASC algorithm unsurprisingly outperforms LP in both data sets, 
since it is designed to exploit the particular structure of 
Problem \ref{prb:problemdef}.

%=========================================================
%           OBJECT RECOGNITION FROM MULTIPLE IMAGE SETS
%=========================================================
\subsection{Object recognition from multi-view image sets}\label{sec:objrec}

In this section we evaluate our graph-based algorithm in the
context of object recognition from multi-view image sets. In this case, the
different views are considered as multiple observations of
the same object, and the problem is to recognize correctly this object.

The proposed MASC method implements Gaussian weights (\ref{eq:Gaussian}) and sets
$k=5$ in the construction of the $k$-NN graph. %In this case we do not use any virtual samples.
We compare MASC to well-known methods from the literature, which mostly gather 
algorithms based on either subspace analysis or density estimation (statistical methods):

\begin{itemize}

\item MSM. The Mutual Subspace Method
\cite{FukuiYama.03,YaFuMa.98}, which is the most well known
representative of the subspace analysis methods. It represents
each image set by a subspace spanned by the principal components,
i.e., eigenvectors of the covariance matrix. The comparison of a
test image set with a training one is then achieved by computing
the \textit{principal angles} \cite{GVL-book} between the two
subspaces. In our experiments, the number of principal components
has been set to nine, which has been found to provide the best
performance.

\item KMSM. MSM has been extended to its nonlinear version
called the Kernel Mutual Subspace Method (KMSM) \cite{SakMuk.00}, in order to take into account the nonlinearity of typical image sets. The main difference of KMSM from MSM is that the images are first nonlinearly mapped into a high dimensional feature space,
before modeling by linear subspaces takes place. In other words, KMSM uses
kernel PCA instead of PCA in order to capture the nonlinearities in the data.
%For the sake of completeness, we include KMSM in our comparisons. 
In KMSM, we use the Gaussian kernel $k(x,y) = \exp(-\frac{\|x-y\|^2}{2 \sigma^2})$, where $\sigma$
is determined exactly in the same way as in the Gaussian weights of our MASC
method.

\item KLD. The KL-divergence algorithm by Shakhnarovich et al
\cite{ShakFishDarrel.02} is the most popular representative
of density-based statistical methods. It formulates the classification from
multiple images as a statistical hypothesis testing problem.
Under the i.i.d and the Gaussian assumptions on the image sets,
the classification problem typically boils down to a computation of 
the KL divergence between sets, which can be computed in closed form in this case. 
The energy cut-off, which determines the number of principal components 
used in the regularization of the covariance matrices, has been set to 0.96.
\end{itemize}

\begin{figure*}[tbh]
\begin{center}
\mbox{
     \subfigure[ETH-80]{\label{fig:eth80all}\includegraphics[width=3.7in]{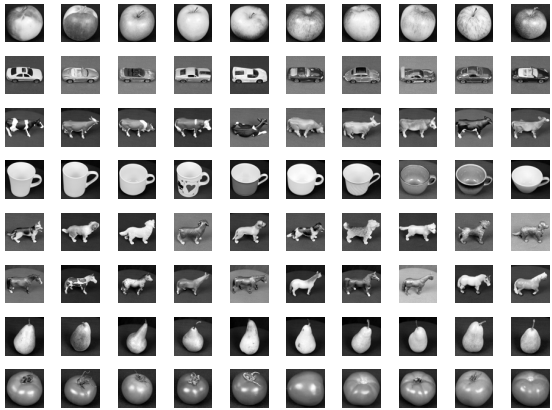}}
     \subfigure[41 views of a sample car model]{\label{fig:eth80:car}\includegraphics[width=3.7in]{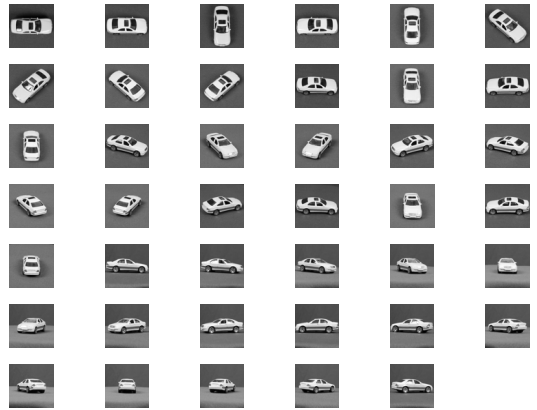}}
     }
\end{center}
\caption{Sample images from the ETH-80 database.} \label{fig:ETH80}
\end{figure*}

In our evaluation, we use the ETH-80 image set \cite{LeibeSchiele.03}, which 
contains 80 object classes from 8 categories; apple, car, cow, cup, dog, horse, 
pear and tomato. Each category has 10 object classes (see Fig. \ref{fig:eth80all}).
Each object class then consists of 41 views of the object spaced evenly over 
the upper viewing hemisphere. Figure \ref{fig:eth80:car} shows the 41 views
from a sample car object class. We use the \texttt{cropped-close128} part of the database. All provided images are of size
128$\times$128 and they are cropped, so that they contain only the object without 
any border area. We downsampled the images to size 32$\times$32 for computational
ease. No further preprocessing is done.

The 41 views from each object class are split randomly into 21 training and 20 test 
samples. In this case, the 20 different views in the test set correspond to the multiple observations 
of the test object. We perform 10 random experiments where the images are randomly split into
training and test sets. Table \ref{tbl:objrecresults} shows the average object recognition rate
for each method. We also report the standard deviation of each method in parentheses. Notice
that the subspace methods are superior to the KLD method which assumes Gaussian distribution
of the data. Notice also that as one would expect, KMSM outperforms MSM that falls short of capturing the
nonlinearities in the data. Finally, observe that our graph-based method clearly outperforms its 
competitors, as it is able to capture not only the nonlinearity but also the manifold structure
of the data.

\begin{table}[t]
\begin{center}
\begin{tabular}{||c||c|c|c||}  \hline
MASC           &     MSM         &   KMSM          &  KLD          \\ \hline
88.88 (1.71)   &    74.88 (5.02) &  83.2500 (3.4)  & 52.5 (3.95)   \\ \hline
\end{tabular}
\end{center}
\caption{Object recognition rate in the mean(std) format, measured on the ETH-80 database.} \label{tbl:objrecresults}
\end{table}

%=========================================================
%           FACE RECOGNITION FROM VIDEO
%=========================================================

\section{Video-based face recognition}
\label{sec:videoface}

%\subsection{Face recognition from video}

\subsection{Experimental setup}\label{sec:expvideo}

In this section we evaluate our graph-based algorithm in the
context of face recognition from video sequences. In this case, the
different video frames are considered as multiple observations of
the same person, and the problem consists in the correct
classification of this person. We evaluate in this section the
behavior of the MASC algorithm in realistic conditions, i.e.,
under variations in head pose, facial expression and illumination.
Note in passing that our algorithm does not assume any temporal
order between the frames; hence, it is also applicable to the
generic problem of face recognition from image sets.

We use two publically available databases; the VidTIMIT
\cite{VidTIMITdb} and the first subset of the Honda/UCSD
\cite{HondaDB} database.  The VidTIMIT database\footnote{http://users.rsise.anu.edu.au/$\sim$conrad/vidtimit/}
contains 43 individuals and there are three face sequences
obtained from three different sessions per subject. The data set
has been recorded in three sessions, with a mean delay of seven days
between session one and two, and six days between session two and
three. In each video sequence each person performed a head
rotation sequence. In particular, the sequence consists of the
person moving his/her head to the left, right, back to the center,
up, then down and finally return to center.

The Honda/UCSD database\footnote{http://vision.ucsd.edu/~leekc/HondaUCSDVideoDatabase/HondaUCSD.html}
contains 59 sequences of 20 subjects. In contrast to the previous
database, the individuals move their head freely, in different
speed and facial expressions. In each sequence, the subjects
perform free in-plane and out-of-plane head rotations. Each person has between 2
and 5 video sequences and the number of sequences per subject is
variable.

For preprocessing, in both databases, we used first P. Viola's
face detector \cite{Viola-facedet} in order to automatically
extract the facial region from each frame. Note that this
typically results in misaligned facial images. Next, we
downsampled the facial images to size 32$\times$32 for
computational ease. No further preprocessing has been performed,
which brings our experimental setup closer to real testing
conditions.

\subsection{Classification results on VidTIMIT}

We first study the performance of the MASC algorithm with the VidTIMIT database. Figure \ref{fig:VidTIMITfaces} shows a few representative images
from a sample face manifold in the VidTIMIT database. Observe the
presence of large head pose variations. Figure \ref{fig:manifold}
shows the 3D projection of the manifold that is obtained using the
ONPP method \cite{KokSaad.07}, which has been shown to be an
effective tool for data visualization. Notice the four clusters
corresponding to the four different head poses i.e., looking left,
right, up and down. This indicates that a graph-based method
should be able to capture the geometry of the manifold and
propagate class labels based on the manifold structure.

Since there are three sessions, we use the following metric for
evaluating the classification performances
\begin{equation}\label{eq:metric}
\overline{e} = \frac{1}{6} \sum_{i=1}^3 \sum_{j=1, j \neq i}^3
e(i,j),
\end{equation}
where $e(i,j)$ is the classification error rate when the $i$th
session is used as training set and the $j$th session is used as
test set. In other words, $\overline{e}$ is the average classification
error rate calculated over the following six experiments, namely
(1,2), (2,1), (1,3), (3,1), (2,3) and (3,2).

\begin{figure}[t]
\begin{center}
\mbox{
     \subfigure[pose 1]{\label{fig:pose1}\includegraphics[width=0.8in]{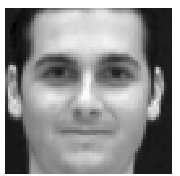}}
     \subfigure[pose 2]{\label{fig:pose2}\includegraphics[width=0.8in]{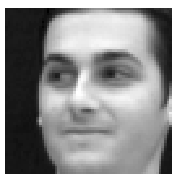}}
     \subfigure[pose 3]{\label{fig:pose3}\includegraphics[width=0.8in]{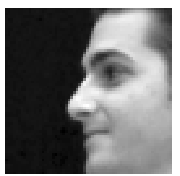}}
     \subfigure[pose 4]{\label{fig:pose4}\includegraphics[width=0.8in]{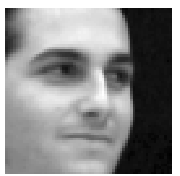}}
     }\\
     \mbox{
     \subfigure[pose 5]{\label{fig:pose5}\includegraphics[width=0.8in]{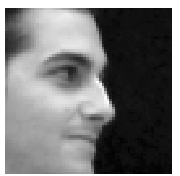}}
     \subfigure[pose 6]{\label{fig:pose6}\includegraphics[width=0.8in]{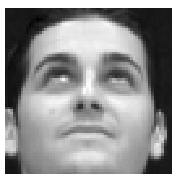}}
     \subfigure[pose 7]{\label{fig:pose7}\includegraphics[width=0.8in]{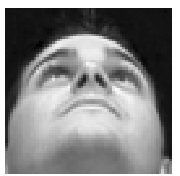}}
     \subfigure[pose 8]{\label{fig:pose8}\includegraphics[width=0.8in]{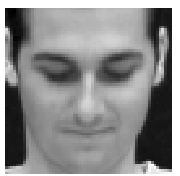}}
     }
\end{center}
\caption{Head pose variations in the VidTIMIT database.}
\label{fig:VidTIMITfaces}
\end{figure}

\begin{figure}[t]
\begin{center}
\includegraphics[width=3in]{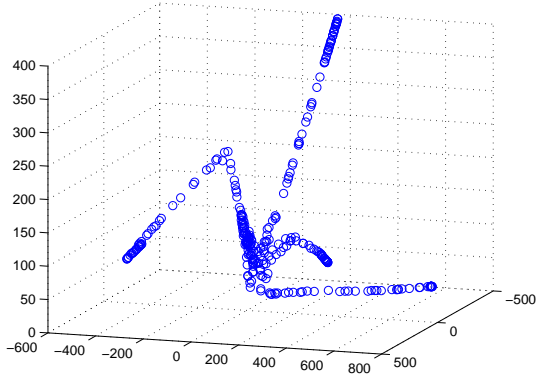}
\end{center}
\caption{A typical face manifold from the VidTIMIT database.
Observe the four clusters corresponding to the four different head
poses (face looking left, right, up and down).}
\label{fig:manifold}
\end{figure}

% \begin{figure*}[t]
% \begin{center}
% \mbox{
%      \subfigure[$r=4$]{\label{fig:steptrain4}\includegraphics[width=3in]{FIGS/experiment14_steptrain4.eps}}
%      \subfigure[$r=8$]{\label{fig:steptrain8}\includegraphics[width=3in]{FIGS/experiment14_steptrain8.eps}}
%      }
% \mbox{
%      \subfigure[$r=12$]{\label{fig:steptrain12}\includegraphics[width=3in]{FIGS/experiment14_steptrain12.eps}}
%      \subfigure[$r=16$]{\label{fig:steptrain16}\includegraphics[width=3in]{FIGS/experiment14_steptrain16.eps}}
%      }
% \end{center}
% \caption{Video face recognition results on the VidTIMIT database.}
% \label{fig:videoresults}
% \end{figure*}

\begin{table}[tbh]
\begin{center}
\begin{tabular}{||c||c|c|c|c||}  \hline
Recognition rate (\%)   &  MASC    &     MSM      &     KMSM       &  KLD      \\ \hline
$r = 4$           &    96.51       &     91.47    &    95.74       &    84.5       \\
$r = 8$           &    96.51       &     87.21    &    94.19       &    81.4       \\
$r = 12$          &    94.96       &     85.66    &    92.64       &    77.52       \\ 
$r = 16$          &    93.8        &     81.4     &    89.15       &    72.48       \\ \hline
\end{tabular}
\end{center}
\caption{Video face recognition results on the VidTIMIT database.} 
\label{tbl:videofacerecresultsVidtimit}
\end{table}

\begin{figure}[t]
\begin{center}
\includegraphics[width=3in]{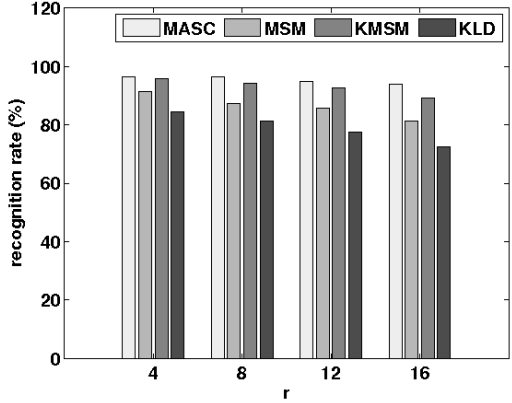}
\end{center}
\caption{Video face recognition results on the VidTIMIT database.}
\label{fig:videoresults}
\end{figure}

We evaluate the video face recognition performance of all methods
for diverse sizes of the training and test sets. The objective is to
assess the robustness of the methods with respect to the size of
the training and test set. For this reason, each image set is re-sampled as
\begin{eqnarray}\nonumber
X_{i,r} &=& X_i(:,1:r:n), ~i=1,\dots,c. 
%X_{\rm test}^{(i)}  &=& X_i(:,1:r:n), ~i=1,\dots,c.
\end{eqnarray}
In the above, the image set $X_i$ is re-sampled with step $r$, i.e., only one image 
every $r$ images is kept. In our experiments, we use different values of $r$ ranging
from 4 to 16 with step 4. For each value of $r$, we measure the average classification error rate according to the relation (\ref{eq:metric}).

Table \ref{tbl:videofacerecresultsVidtimit} 
shows the recognition performance, for $r$ ranging from 4 to 16 with step 4.
Figure \ref{fig:videoresults} shows graphically the same results. 
Observe that the KLD method that relies on density estimation is
sensitive to the number of the available data. Also, notice that
MSM is superior to KLD, which is expected since KLD relies on the 
imprecise assumption that data follow a Gaussian distribution. 
Furthermore, KMSM, the nonlinear variant of MSM, outperforms the latter
that has trouble in capturing the nonlinear structures in the data. Finally, we
observe that MASC clearly outperforms its competitors in the vast
majority of cases. At the same time, it stays robust to significant 
re-sampling of the data, since its performance remains
almost the same for each value of $r$.

\subsection{Classification results on Honda/UCSD}

\begin{figure}[tbh]
\begin{center}
\mbox{
     \subfigure[pose 1]{\label{fig:hondapose1}\includegraphics[width=0.8in]{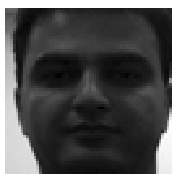}}
     \subfigure[pose 2]{\label{fig:hondapose2}\includegraphics[width=0.8in]{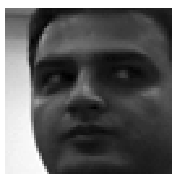}}
     \subfigure[pose 3]{\label{fig:hondapose3}\includegraphics[width=0.8in]{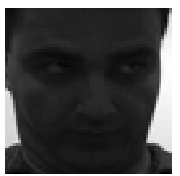}}
     \subfigure[pose 4]{\label{fig:hondapose4}\includegraphics[width=0.8in]{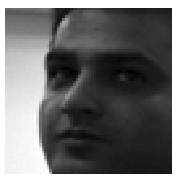}}
     }\\
     \mbox{
     \subfigure[pose 5]{\label{fig:hondapose5}\includegraphics[width=0.8in]{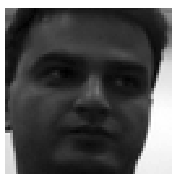}}
     \subfigure[pose 6]{\label{fig:hondapose6}\includegraphics[width=0.8in]{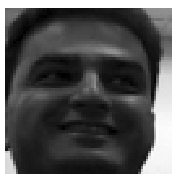}}
     \subfigure[pose 7]{\label{fig:hondapose7}\includegraphics[width=0.8in]{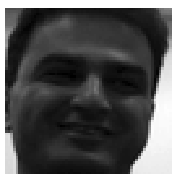}}
     \subfigure[pose 8]{\label{fig:hondapose8}\includegraphics[width=0.8in]{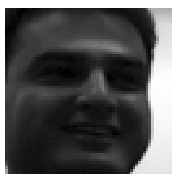}}
     }
\end{center}
\caption{Head pose variations in the Honda/UCSD database.}
\label{fig:Hondafaces}
\end{figure}

\begin{figure}[tbh]
\begin{center}
\includegraphics[width=3in]{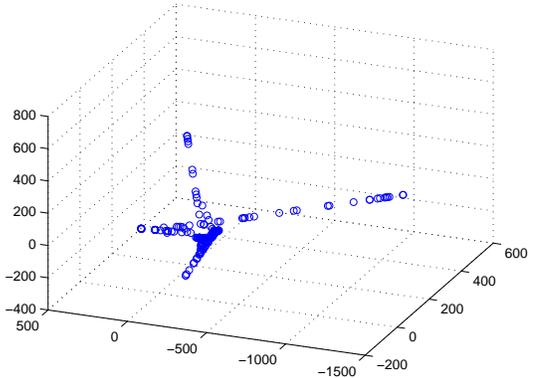}
\end{center}
\caption{A typical face manifold from the Honda/UCSD database.}
\label{fig:hondamanifold}
\end{figure}

\begin{figure}[tbh]
\begin{center}
\includegraphics[width=3.5in]{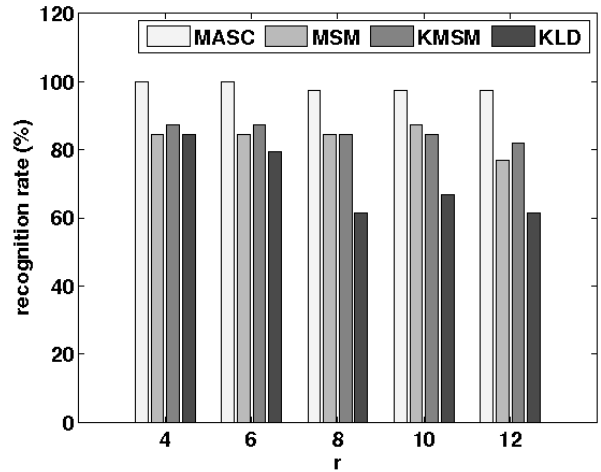}
\end{center}
\caption{Video face recognition results on the Honda/UCSD
database.} \label{fig:hondavideoresults}
\end{figure}

\begin{table}[tbh]
\begin{center}
\begin{tabular}{||c||c|c|c|c||}  \hline
Recognition rate (\%)   &  MASC   &     MSM  &  KMSM     &  KLD        \\ \hline
$r = 4$           &  100    &    84.62       &  87.18  &    84.62    \\
$r = 6$           &  100    &    84.62       &  87.18  &    79.49    \\
$r = 8$           &  97.44  &    84.62       &  84.62  &    61.54    \\
$r = 10$          &  97.44  &    87.18       &  84.62  &    66.67    \\
$r = 12$          &  97.44  &    76.92       &  82.05  &    61.54    \\ \hline
\end{tabular}
\end{center}
\caption{Video face recognition results on the Honda/UCSD
database.} \label{tbl:videofacerecresultsHonda}
\end{table}

We further study the video-based face recognition performance on
the Honda/UCSD database. Figure \ref{fig:Hondafaces} shows a few
representative images from a sample face manifold in the
Honda/UCSD database. Observe the presence of large head pose
variations along with facial expressions. The projection of the
manifold on the 3D space using ONPP shows again clearly the
manifold structure of the data (see Figure
\ref{fig:hondamanifold}), which implies that a graph-based method
is more suitable for such kind of data.

The Honda/UCSD database comes with a default splitting into training and test
sets, which contains 20 training and 39 test video sequences. We
use this default setup and we report the classification
performance of all methods, under different data re-sampling rates. 
Similarly as above, both training and test image sets are re-sampled 
with step $r$, i.e., $X_{i,r} = X_i(:,1:r:n),~i=1,\ldots,c$. Table
\ref{tbl:videofacerecresultsHonda} shows the recognition
rates, when $r$ varies from 4 to 12 with step 2. Figure
\ref{fig:hondavideoresults} shows the same results graphically.
Recall that larger values of $r$ imply sparser image sets. Observe again
that KLD is mostly affected by $r$, by suffering loss in
performance. This is not surprising since it is a density-based
method and densities cannot be accurately estimated (in general) with a few samples. 
MSM seems to be more robust, yielding better
results than KLD, but as expected, it is inferior to KMSM in the majority of cases. 
Finally, MASC is again the best performer and it exhibits very high robustness 
against data re-sampling.

Regarding the relative performance of MASC and KMSM, we should finally
stress out that KMSM is a kernel technique that attempts to capture 
the nonlinear structure of the data by assuming a linear model after applying a nonlinear
mapping of the data into a high dimensional space. Although this methodology stays generic
and presents certain advantages, it is still not clear whether it is capable of capturing the individual (e.g., manifold) structure of diverse data sets. 
%and its performance becomes dependent on the particular characteristics of the nonlinear mapping under use.  
On the other hand, the MASC method explicitly relies on a graph model that may fit 
much better the manifold structure of the data. 
Furthermore, it provides a way to cope with the curse of
dimensionality, since the intrinsic dimension of the manifolds is 
typically very small. We believe that graph methods have a great
potential in this field.

\subsection{Video-based face recognition overview}\label{sec:relatedworkfacevideo}

For the sake of completeness, we review briefly in this last
section the state of the art in video-based face recognition.
Typically, one may distinguish between two main families of
methods; those that are based on subspace analysis and those that
are based on density estimation (statistical methods). The most
representative methods for these two families are respectively the
MSM \cite{FukuiYama.03,YaFuMa.98} and KMSM \cite{SakMuk.00} methods and the solution based on
KLD \cite{ShakFishDarrel.02}, which have been used in the
experiments above.

%====== subspace methods ======================
Among the methods based on subspace analysis, we should mention
the extension of principal angles from subspaces, to nonlinear
manifolds. In a recent article \cite{WangShanChenGao.08} it was
proposed to represent the facial manifold by a collection of
linear patches, which are recovered by a non-iterative algorithm
that augments the current patch until the linearity criterion is
violated. This manifold representation allows for defining the
distance between manifolds as integration of distances between
linear patches. For comparing two linear patches, the authors
propose a distance measure that is a mixture between (i) the
principal angles and (ii) exemplar-based distance. However, it is
not clearly justified why such a mixture is needed and what is the
relative benefit over the individual distances. Moreover, their
proposed method requires the computation of both geodesic and
Euclidean distances as well as setting four parameters. On the
contrary, our MASC method needs only one parameter ($k$) to be set
and it requires the computation of the Euclidean distances only.
Note finally that their method achieves comparable results with
MASC on the Honda/UCSD database, but at a higher computational
cost and at the price of tuning four parameters.

Along the same lines, the authors in \cite{KimArandCip.07} propose
a similarity measure between manifolds that is a mixture of
similarity between subspaces and similarity between local linear
patches. Each individual similarity is based on a weighted
combination of principal angles and those weights are learnt by
AdaBoost for improved discriminative performance. In contrast to
the previous paper \cite{WangShanChenGao.08}, the linear patches
are extracted here using mixtures of Probabilistic PCA (PPCA).
PPCA mixture fitting is a highly non-trivial task, which requires
an estimate of the local principal subspace dimension and it also
involves model selection. This step is quite computationally intensive, as noted in \cite{WangShanChenGao.08}.

%====== statistical methods ======================
The main limitation of the statistical methods such as KLD
\cite{ShakFishDarrel.02} is the inadequacy of the Gaussianity
assumption of face images sets; face sequences rather have a
manifold structure. The test video frames are moreover not
independent, so that the i.i.d assumption is unrealistic as well.
The authors in \cite{ArandjSFCD.05} therefore extend the work of
KL divergence by replacing the Gaussian densities by Gaussian
Mixture Models (GMMs), which provides a more flexible method for
density estimation. However, the KL divergence in this case cannot
be computed in a closed form, which makes the authors to resort to
Monte Carlo simulations that are quite computationally intensive.

Finally, there have been a few other methods that cannot be
directly categorized in the above families of methods. The authors
in \cite{ZhouChellappa.06} propose ensemble similarity metrics
that are based on probabilistic distance measures, evaluated in
Reproducing Kernel Hilbert spaces. All computations are performed
under the Gaussianity assumption, which is unfortunately not
realistic for facial manifolds.

In \cite{ZhouChellappa.04}, the authors provide a probabilistic
framework for face recognition from image sets. They model the
identity as a discrete or continuous random variable and they
provide a statistical framework for estimating the identity by
marginalizing over face localization, illumination and head pose.
Illumination-invariant basis vectors are learnt for each
(discretized) pose and the resulting subspace is used for
representing the low dimensional vector that encodes the subject
identity. However, the statistical framework requires the
computation of several integrals that are numerically
approximated. Also, the proposed method assumes that training
images are available for every subject at each possible pose and
illumination, which is hard to satisfy in practice.

X. Liu and T. Chen in \cite{LiuChen.03} proposed a methodology based on adaptive hidden Markov models for video-based face 
recognition. The temporal dynamics of each subject are learnt during
training and subsequently used for recognition. However, the proposed 
approach assumes temporal order of the frames in the face sequence and unfortunately 
it is not applicable to the more generic problem of recognition from
image sets. The study in \cite{HadPiet.04} further investigates how the
performance of the above approach is affected by the face sequence length 
and the image quality.

%Finally, a manifold-based method is proposed in
%\cite{ArandCipolla.06}, where the main idea is to re-illuminate
%the test sequence while comparing it with each individual gallery
%sequence. A PPCA on the manifold representation is finally used to
%recognize test sequences based on the maximum likelihood
%principle. However, when comparing manifolds, correspondences
%between frames are established based on genetic algorithms, which
%is computationally expensive. This frame correspondence is used to
%re-illuminate the test sequence that is subsequently classified by
%maximum likelihood principle. The proposed framework also employs
%an off-line learning process based on ample training data, which
%are used for deriving an albedo-free image representation and then
%GMM is fitted in this representation.

\section{Conclusions}

In this paper we have addressed the problem of classification of
multiple observations of the same object. We have proposed to
exploit the specific structure of this problem in a graph-based algorithm
inspired by label propagation. The graph-based algorithm
relies on the smoothness assumption of the manifold in order to
learn the unknown label matrix, under the constraint that all
observations correspond to the same class. We have formulated this
process as a discrete optimization problem that can be solved
efficiently by a low complexity algorithm.

We provide experimental results that illustrate the performance of
the proposed solution for the classification of handwritten
digits, for object recognition and for video-based face recognition. In the two latter cases,
the graph-based solution outperforms state-of-the-art methods on
three publically available data sets. This clearly outlines the
potential of the proposed graph-based solution that is able to
advantageously capture the structure of image manifolds.

\bibliographystyle{unsrt}
\bibliography{../BIB/kokiopou}

\end{document}